\documentclass{article}

\usepackage[preprint]{neurips_2025}
\usepackage{natbib}

\usepackage[utf8]{inputenc} 
\usepackage[T1]{fontenc}    
\usepackage{hyperref}       
\usepackage{url}            
\usepackage{booktabs}       
\usepackage{amsfonts}       
\usepackage{nicefrac}       
\usepackage{microtype}      
\usepackage{xcolor}         
\usepackage{subcaption}
\usepackage[linesnumbered,ruled]{algorithm2e}
\usepackage{wrapfig}  

\usepackage{amsmath}
\usepackage{amssymb}
\usepackage{mathtools}
\usepackage{amsthm}
\usepackage{thm-restate}
\usepackage{color-edits}
\addauthor{ev}{blue}

\theoremstyle{plain}
\newtheorem{theorem}{Theorem}[section]
\newtheorem{proposition}[theorem]{Proposition}
\newtheorem{lemma}[theorem]{Lemma}

\theoremstyle{definition}
\newtheorem{definition}[theorem]{Definition}

\theoremstyle{remark}

\newenvironment{hproof}{%
  \proof}{\endproof}

\title{Identity-Link IRT for Label-Free LLM Evaluation: Preserving Additivity in TVD-MI Scores}
\author{
  Zachary Robertson \\
  Computer Science\\
  Stanford University\\
  \texttt{zroberts@stanford.edu}
}
\date{\today}

\begin{document}

\maketitle

\begin{abstract}
Pairwise comparisons of large language models using total variation distance mutual information (TVD-MI) produces binary critic decisions per pair. We show that averaging TVD-MI’s binary trials yields centered-probability scores with additive structure suitable for item-response theory (IRT) without nonlinear link functions. Maximum likelihood approaches to IRT use logistic links, but we find empirically that these transformations introduce curvature that breaks additivity: across three domains, the identity link yields median curl on raw data of 0.080–0.150 (P95=0.474–0.580), whereas probit/logit introduce substantially higher violations (median [0.245,0.588], P95[0.825,2.252]). We derive this clipped-linear model from Gini entropy maximization, yielding a box-constrained least-squares formulation that handles boundary saturation. At 33\% coverage, we achieve holdout RMSE 0.117 ± 0.008 while preserving agent rankings (Spearman $\rho$=0.972 ± 0.015), 3$\times$ fewer evaluations than full dense. Judge robustness analysis (GPT-4o-mini vs Llama3-70b) shows strong agreement in agent rankings ($\rho$=0.872) and consistent identity link advantage. TVD-MI's geometry is best preserved by identity mapping for efficient LLM evaluation, applicable to other bounded-response domains.
\end{abstract}

\section{Introduction}

Evaluating large language models (LLMs) at scale is expensive: dense evaluation matrices grow quadratically in models and items. Standard practice—pairwise preferences or mutual-information-based comparisons—effectively requires scoring every agent–item pair, which quickly becomes prohibitive \citep{chiang2024chatbot}. In \emph{peer-prediction} settings with no ground truth, one must rely on inter-model agreement patterns \citep{prelec2004bayesian,dasgupta2013crowdsourced, xubenchmarking}.

\citet{robertson2025measure} propose total-variation mutual information (TVD-MI) for label-free LLM evaluation. TV is the unique $f$-divergence that is also an IPM; its supremum is attained by a \emph{binary critic}, so TVD-MI reduces to optimal binary tests distinguishing paired from independent responses \citep{sriperumbudur2009integral,tsybakov2008nonparametric}. Averaging these binary trials yields agent–item scores $S\in[-1,1]^{k\times n}$ with one entry $s_{ij}$ per (agent $i$, item $j$).

Items discriminate among agents similarly to classical \emph{item response theory} (IRT) \citep{rasch1993probabilistic}. Empirically we find that the additivity appears \emph{on the raw TV scale}, without nonlinear links:
\begin{equation}
s_{ij}\approx \theta_i-b_j,
\label{eq:rasch}
\end{equation}
with $\theta_i$ a latent ability and $b_j$ an item difficulty. In contrast, logistic/probit links curve the geometry and break discrete integrability; the remainder of the paper formalizes and exploits this identity-link regime for sparse, sample-efficient evaluation.

\subsection{Our Contributions}

TVD-MI evaluation matrices exhibit near-additive structure that standard IRT link functions distort.

\textbf{Identity link dominates for TVD-based scores.} Across PubMed/OPUS/ICLR, the identity map yields median curl $0.080$–$0.150$ (P95 $=0.474$–$0.580$), whereas probit/logit induce $2$–$7\times$ larger violations (median $[0.245,0.588]$, P95 $[0.825,2.252]$). Bootstrap gaps are significant (CIs exclude zero). Baselines corroborate the geometry: our clipped-linear additive model attains the best reconstruction on PubMed (RMSE $0.1215$), while unconstrained UV factorization is far worse (RMSE $0.196$), demonstrating the necessity of additivity (Section~\ref{sec:baseline}).

\textbf{Clipped-linear model from Gini entropy.} We derive the identity link by projecting TVD-MI scores onto the additive manifold under Gini (Rényi-2) entropy, yielding a box-constrained least-squares estimator that handles saturation. Discrete integrability tests confirm approximate additivity (Section~\ref{sec:model},~\ref{sec:link-ablation}).

\textbf{Sample-efficient sparse recovery.} With $33\%$ coverage and $d$-core~$\ge 3$ connectivity, we achieve holdout RMSE $0.117\pm0.008$ vs.\ $0.111\pm0.006$ for dense, i.e., $\sim\!3\times$ fewer evaluations with near-perfect ranking preservation (Spearman $\rho=0.972\pm0.015$, Ranking AUC $0.967\pm0.012$ vs.\ $0.983\pm0.008$). Results transfer across domains and judges (Section~\ref{sec:results}).

\section{Background and Related Work}

\subsection{Item Response Theory and Matrix Completion}

Item response theory (IRT) models binary/graded responses through latent traits \citep{rasch1993probabilistic}. The Rasch (1PL) model assumes $\text{logit}(P(u_{ij} = 1)) = \theta_i - b_j$, where $\theta_i$ is person ability and $b_j$ is item difficulty. This rank-2 structure has invariance properties and enables recovery from incomplete observations via matrix completion \citep{andersen1977sufficient, candes2012exact}.

Recent work applies IRT to LLM evaluation with ground truth labels. \citet{castleman2025rethinking} use IRT to analyze math benchmarks, revealing that many items provide redundant signal. \citet{zhou2025lost} show that IRT can identify mislabeled or ambiguous test items in NLP benchmarks. Other works propose standardized selection based on IRT analysis \citep{ding2024easy2hard, liu2025leveraging}. However, all prior work assumes access to ground truth—a requirement we eliminate through peer prediction while discovering that TVD-MI naturally produces additive structure without logistic links.

\subsection{Peer Prediction and TVD-MI}

Peer-prediction mechanisms elicit information without ground truth \citep{prelec2004bayesian,dasgupta2013crowdsourced, xubenchmarking}.  \citet{robertson2025measure} introduced TVD mutual information for LLM evaluation without ground truth. For response distributions $(X, Y)$:
\begin{equation}
I_{\text{TVD}}(X;Y) = \text{TV}(P_{XY}, P_X \otimes P_Y),
\end{equation}
where $P_{XY}$ is the joint (paired responses) and $P_X \otimes P_Y$ is the product of marginals (independent sources).

TVD can be estimated through binary classification \citep{tsybakov2008nonparametric}. For any decision rule $r$:
\begin{align}
\text{TPR}_r &:= \Pr_{S \sim P_{XY}}[r(S) = 1], \quad
\text{FPR}_r := \Pr_{S \sim P_X \otimes P_Y}[r(S) = 1]
\end{align}
The difference $\text{TPR}_r - \text{FPR}_r$ lower-bounds TVD-MI, with equality at the optimal critic.

\begin{lemma}[Binary optimal critic for total variation]
\label{lem:binary-critic}
Total variation admits the IPM representation
\[
\mathrm{TV}(P,Q)=\sup_{\|h\|_\infty\le 1}\ \mathbb{E}_P[h]-\mathbb{E}_Q[h],
\]
and the supremum is achieved by a binary critic
$h^*(x)\in\{-1,1\}$ with $h^*(x)=\mathrm{sign}(p(x)-q(x))$ almost everywhere.
Moreover, among $f$-divergences, total variation is (up to a positive multiplicative constant)
the only one that is also an IPM over a symmetric convex function class (the $L_\infty$ unit ball);
hence it is the only $f$-divergence admitting such a bounded, binary optimal critic.
\end{lemma}

This produces agent-item score matrices $S \in [-1,1]^{K \times J}$ where $s_{ij} = \text{TPR}_{ij} - \text{FPR}_{ij}$ for $K$ agents and $J$ items. The original formulation requires $O(K^2 J)$ evaluations.

\subsection{Why Traditional IRT Links Fail for TVD-MI}

Classical IRT models assume binary response data generated from a latent logistic process, motivating the logit link $\ell(p) = \log(p/(1-p))$ \citep{mccullagh2019generalized}. However, TVD-MI scores arise from a different process. \textbf{Pairwise discrimination:} For each item $j$, we compute TVD-MI between all agent pairs, yielding a $K \times K$ matrix of discrimination scores in $[-1,1]$. \textbf{Linear averaging:} Agent $i$'s score on item $j$ is the average discrimination against all other agents: $s_{ij} = \frac{1}{K-1} \sum_{k \neq i} \text{TVD-MI}(i,k|j)$. \textbf{Additivity preservation:} If pairwise scores exhibit additive structure ($\theta_i - \theta_k + b_j$), the average inherits this structure in the raw space. This identity-link estimator aligns with Gini/Brier criterion used in proper scoring \citep{yuan2021gini, glenn1950verification}. We validate this empirically in Section~\ref{sec:link-ablation}, showing that the identity link yields smaller violations than curved links.

\section{Experimental Setup}
\label{sec:setup}

We evaluate 30 agent configurations across three domains: summarization (PubMed, 200 items), translation (OPUS, 186 items), and peer review (ICLR, 100 items). Agents span faithful baselines, stylistic variants, strategic distortions, and low-effort responses (see Appendix~\ref{app:agent-configs} for details).

TVD-MI scores are computed through binary classification: an LLM judge (GPT-4o-mini) distinguishes whether response pairs come from the same agent versus different agents. Agent $i$'s score on item $k$ averages the discriminative signal (TPR - FPR) across all pairwise comparisons with other agents, producing matrices $S \in [-1, 1]^{K \times J}$. We reserve 20\% of agent-item pairs as a holdout set before computing any scores, ensuring train-test independence (see Appendix~\ref{app:tvdmi-construction}).

The resulting matrices exhibit boundary saturation (2-3\% of entries at $\pm 1$) with mean 0.18, SD 0.31, and natural sparsity from computational constraints (see Appendix~\ref{app:data-characteristics}). We use the identity link (no transformation) for our additive model $s_{ij} = \theta_i - b_j$, as Section~\ref{sec:link-ablation} shows it outperforms logistic links by 2--7$\times$ in preserving discrete integrability.

\section{Model and Diagnostics}\label{sec:model}

\subsection{Clipped-Linear Model from Gini Entropy}

We first validate that TVD-MI scores exhibit approximate additive structure through discrete integrability tests (Section~\ref{sec:link-ablation}). Given this empirical finding, we model the scores through an additive decomposition on the raw scale:
\begin{equation}
s_{ij} = \theta_i - b_j + \epsilon_{ij}, \quad |s_{ij}| \leq 1,
\label{eq:additive-model}
\end{equation}
where $\theta_i$ represents agent $i$'s latent discriminative ability, $b_j$ represents item $j$'s difficulty, and the box constraint naturally captures saturation at $\pm 1$. In this model, a sufficient condition for \emph{approximate} additivity is that the noise $\epsilon_{ij}$ is small or weakly dependent.

\begin{proposition}[Quadratic (Gini) projection onto the additive manifold]
\label{prop:gini-projection}
Assume we project scores onto the additive manifold by enforcing the discrete
integrability (rectangle) constraints. Consider
\begin{align}
\min_{s_{ij}} \ & \sum_{i,j} s_{ij}^2 \\
\text{s.t. } & |s_{ij}|\le 1,\quad
s_{ij}=t_{ij}\ \forall (i,j)\in\Omega,\quad
s_{ij}-s_{i'j}=s_{ij'}-s_{i'j'}\ \forall\ (i,i',j,j').
\end{align}
Then the optimizer has the additive form $s_{ij}=\theta_i-b_j$ with $|s_{ij}|\le 1$.
The quadratic objective corresponds to maximizing Gini impurity
$\mathrm{Gini}(p)=1-\sum_x p(x)^2$ under the change of variables $s=2p-1$,
and $\mathrm{TV}(X;X)=\mathrm{Gini}(X)$ links this criterion to total variation.
\end{proposition}  

\begin{hproof}
The objective minimizes $\sum s_{ij}^2$, which corresponds to maximizing Gini impurity $\text{Gini}(p) = 2p(1-p) = \frac{1}{2}(1 - s^2)$ where $s = 2p-1$. Crucially, Gini impurity is the \emph{natural entropy measure for TVD}: Lemma~\ref{lem:tvd-gini} (Appendix~\ref{app:tvd-gini-connection}) shows that $\text{TVD-MI}(X;X) = \text{Gini}(X)$, establishing that we are maximizing the self-information under total variation distance. The rectangle constraints enforce discrete integrability, which by Lemma~\ref{lem:integrability} (Appendix~\ref{app:gini-derivation}) is equivalent to the additive form $s_{ij} = \theta_i - b_j$. The dual problem yields a box-constrained least-squares objective with identity link.
\end{hproof}

This proposition assumes additivity (we study deviations in Section~\ref{sec:link-ablation}) and shows that projecting onto the additive manifold under the natural entropy for TVD yields the identity link. This contrasts with Shannon entropy, which yields logistic links that distort the geometry when data is already approximately additive in bounded space.

Given observed entries $\Omega \subseteq [K] \times [J]$, we estimate parameters via regularized least squares:
\begin{equation}
\min_{\theta, b} \sum_{(i,j) \in \Omega} \big(s_{ij} - (\theta_i - b_j)\big)^2 + \lambda\left(\|\theta\|_2^2 + \|b\|_2^2\right),
\label{eq:rasch-objective}
\end{equation}
with gauge constraint $\sum_j b_j = 0$ for identifiability. The ridge penalty $\lambda(\|\theta\|_2^2 + \|b\|_2^2)$ fixes the scale of the latent variables and provides regularization against overfitting in sparse observation regimes. We use $\lambda = 10^{-6}$ throughout. This convex program admits closed-form updates via alternating minimization, converging to global optimum in $O(|\Omega|)$ time per iteration.

\subsection{Discrete Integrability Test}

The additive model assumes that scores exhibit \emph{discrete integrability}—the mixed second differences vanish. We formalize this through a rectangle test:

\begin{definition}[Rectangle Deviation]
For any rectangle of agents $(i, i')$ and items $(j, j')$, the discrete integrability violation is:
\begin{equation}
\Delta(i,i',j,j') := s_{ij} - s_{i'j} - s_{ij'} + s_{i'j'}.
\label{eq:curl}
\end{equation}
Under perfect additivity, $\Delta = 0$ for all rectangles.
\end{definition}

The curl measures deviation from the additive manifold. If $s_{ij} = \theta_i - b_j$, then:
\begin{align}
\Delta &= (\theta_i - b_j) - (\theta_{i'} - b_j) - (\theta_i - b_{j'}) + (\theta_{i'} - b_{j'})\\
&= \theta_i - \theta_{i'} - \theta_i + \theta_{i'} = 0.
\end{align}

We empirically evaluate curl magnitude by sampling random rectangles from the observed data.

\subsection{Link and Model Ablation}
\label{sec:link-ablation}

We computed discrete integrability violations to test the identity link choice on the \emph{raw score matrices} across three domains using paired rectangles (20,000 per condition). Figure~\ref{fig:curl-wow} shows empirical cumulative distribution functions of curl magnitude $|\Delta| = |s_{ij} - s_{i'j} - s_{ij'} + s_{i'j'}|$ for identity, probit, and logit link-transformed data. 

\textbf{Statistical methodology.} We ensured fair comparison by using identical rectangle samples across all links within each domain. Statistical significance was assessed via bootstrap resampling (500 iterations): for each bootstrap sample, we resampled agents and items with replacement, then computed median curl on the resampled matrix with newly drawn rectangles. This approach avoids dependence issues that arise from sampling individual rectangles, which share entries and violate independence assumptions.

Table~\ref{tab:data-curl} shows that identity transformation (i.e., no transformation) achieves the lowest median curl across all domains: 0.129 (PubMed), 0.150 (OPUS), and 0.080 (ICLR). Probit transformation increases median curl by 90–255\% (0.245–0.286), while logit produces 2–7× higher violations (0.438–0.588). The 95th percentiles reveal heavier tails: identity P95 [0.474, 0.580] versus logit P95 [1.454, 2.252].

\begin{table}[ht]
\centering
\caption{Data-space curl on raw score matrices: median and P95 $|\Delta|$ (lower is better). Identity preserves natural additivity; probit/logit introduce substantial violations. All differences significant via bootstrap (CIs exclude zero).}
\label{tab:data-curl}
\begin{tabular}{lcccccc}
\toprule
& \multicolumn{2}{c}{Identity} & \multicolumn{2}{c}{Probit} & \multicolumn{2}{c}{Logit} \\
\cmidrule(lr){2-3} \cmidrule(lr){4-5} \cmidrule(lr){6-7}
Domain & Median & P95 & Median & P95 & Median & P95 \\
\midrule
PubMed & \textbf{0.129} & \textbf{0.474} & 0.245 & 0.825 & 0.438 & 1.454 \\
OPUS   & \textbf{0.150} & \textbf{0.492} & 0.268 & 0.862 & 0.479 & 1.547 \\
ICLR   & \textbf{0.080} & \textbf{0.580} & 0.286 & 1.188 & 0.588 & 2.252 \\
\midrule
Mean   & \textbf{0.120} & \textbf{0.515} & 0.266 & 0.958 & 0.502 & 1.751 \\
\bottomrule
\end{tabular}
\end{table}

\textbf{Bootstrap confidence intervals.} Table~\ref{tab:curl-bootstrap} shows that all pairwise differences are statistically significant. Identity vs probit differences range from $\Delta \in [-0.112, -0.213]$ (all CIs exclude zero), while identity vs logit differences are even larger ($\Delta=[-0.298,-0.528]$). The consistent pattern across domains supports that TVD-MI's inherent geometry favors identity mapping.

\begin{table}[ht]
\centering
\caption{Bootstrap 95\% confidence intervals for median curl differences (500 iterations). All comparisons show statistically significant advantages for identity link.}
\label{tab:curl-bootstrap}
\begin{tabular}{lccc}
\toprule
Domain & Identity vs Probit & Identity vs Logit & Probit vs Logit \\
\midrule
PubMed & -0.112 [-0.122, -0.097] & -0.298 [-0.326, -0.256] & -0.185 [-0.204, -0.159] \\
OPUS   & -0.117 [-0.128, -0.101] & -0.322 [-0.346, -0.278] & -0.204 [-0.219, -0.176] \\
ICLR   & -0.213 [-0.228, -0.172] & -0.528 [-0.560, -0.426] & -0.315 [-0.335, -0.250] \\
\bottomrule
\end{tabular}
\end{table}

The identity link preserves TVD-MI's natural additivity, while monotone transformations (probit, logit) distort this geometry even when applied to bounded data. 

\begin{figure}[t]
\centering
\includegraphics[width=0.75\linewidth]{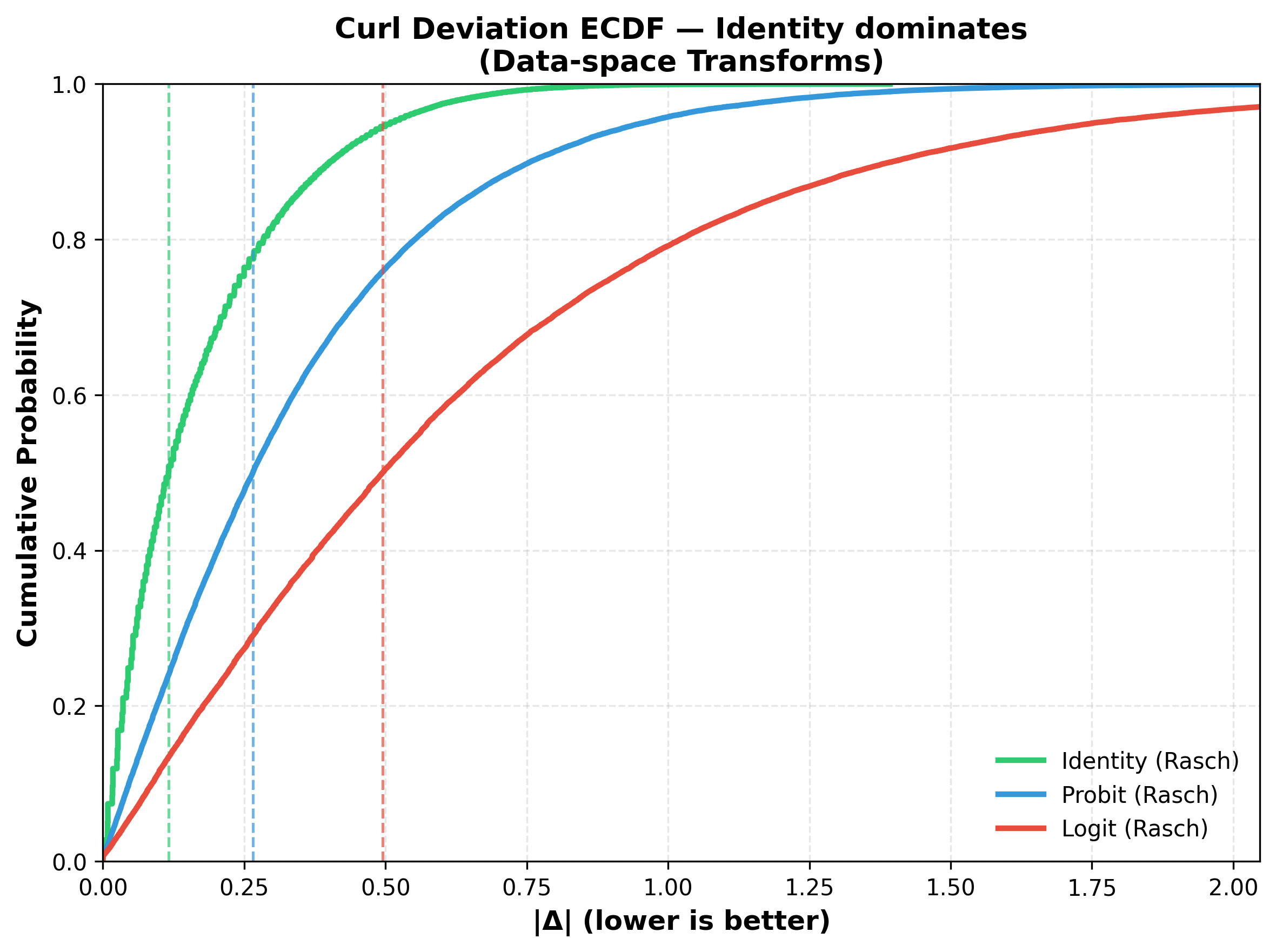}
\caption{Identity link preserves natural additivity of TVD-MI scores. Empirical cumulative distribution functions of curl magnitude $|\Delta|$ on raw score matrices across three domains (20,000 paired rectangles per condition). Identity is leftmost with median curl 0.080–0.150, while curved links shift right with 2–7× higher violations (logit median: 0.438–0.588). TVD-MI averaging produces inherently additive data in $[-1,1]$ space; nonlinear links warp this geometry.}
\label{fig:curl-wow}
\end{figure}

\subsection{Connectivity Requirements}

The observation pattern $\Omega$ must satisfy connectivity constraints for reliable recovery. \textbf{Minimum degree constraint:} Every agent observes $\geq d$ items and every item is observed by $\geq d$ agents. \textbf{Single component:} The bipartite graph formed by $\Omega$ is connected. Empirically, we find $d=3$ sufficient for stable recovery, providing redundancy against individual noisy measurements while maintaining sparsity. The connectivity ensures that all parameters are anchored to the same gauge, preventing drift between disconnected components. We refer to this as the "$d$-core $\geq 3$" constraint to avoid confusion with the number of agents $K$.

\section{Sampling and Protocol}

We evaluate sparse recovery through a train-test protocol with 20\% holdout and four sampling strategies: row sampling (agents observe $\alpha$-fraction of items), column sampling (items evaluated by $\beta$-fraction of agents), hybrid sampling (independent pair selection), and efficient $(n\log n)$ sampling ($|\Omega| = C(K+J)\log(K+J)$ pairs). All regimes enforce $d$-core $\geq 3$ connectivity. We measure reconstruction accuracy (holdout RMSE) and ranking fidelity (Spearman $\rho$, Kendall $\tau$, Ranking AUC). Uncertainty is quantified through bootstrap resampling (500 iterations). Ranking AUC is constructed from the (model) scores averaged over items i.e. one score per agent used for classification. See Appendix~\ref{app:sampling-protocol} for full details.

\section{Results}
\label{sec:results}

\subsection{Sample Efficiency}

Figure~\ref{fig:sparsity-curve} shows hold-out RMSE as a function of training coverage across different sampling regimes. The dense baseline (80\% training coverage) achieves RMSE of 0.111. The $(n\log n)$ regime with $C=1.6$ achieves RMSE of 0.117 at 33\% coverage, a 3× reduction in required evaluations with 5.4\% relative error increase.

\begin{figure}[ht]
\centering
\begin{subfigure}{0.55\linewidth}
  \includegraphics[width=\linewidth]{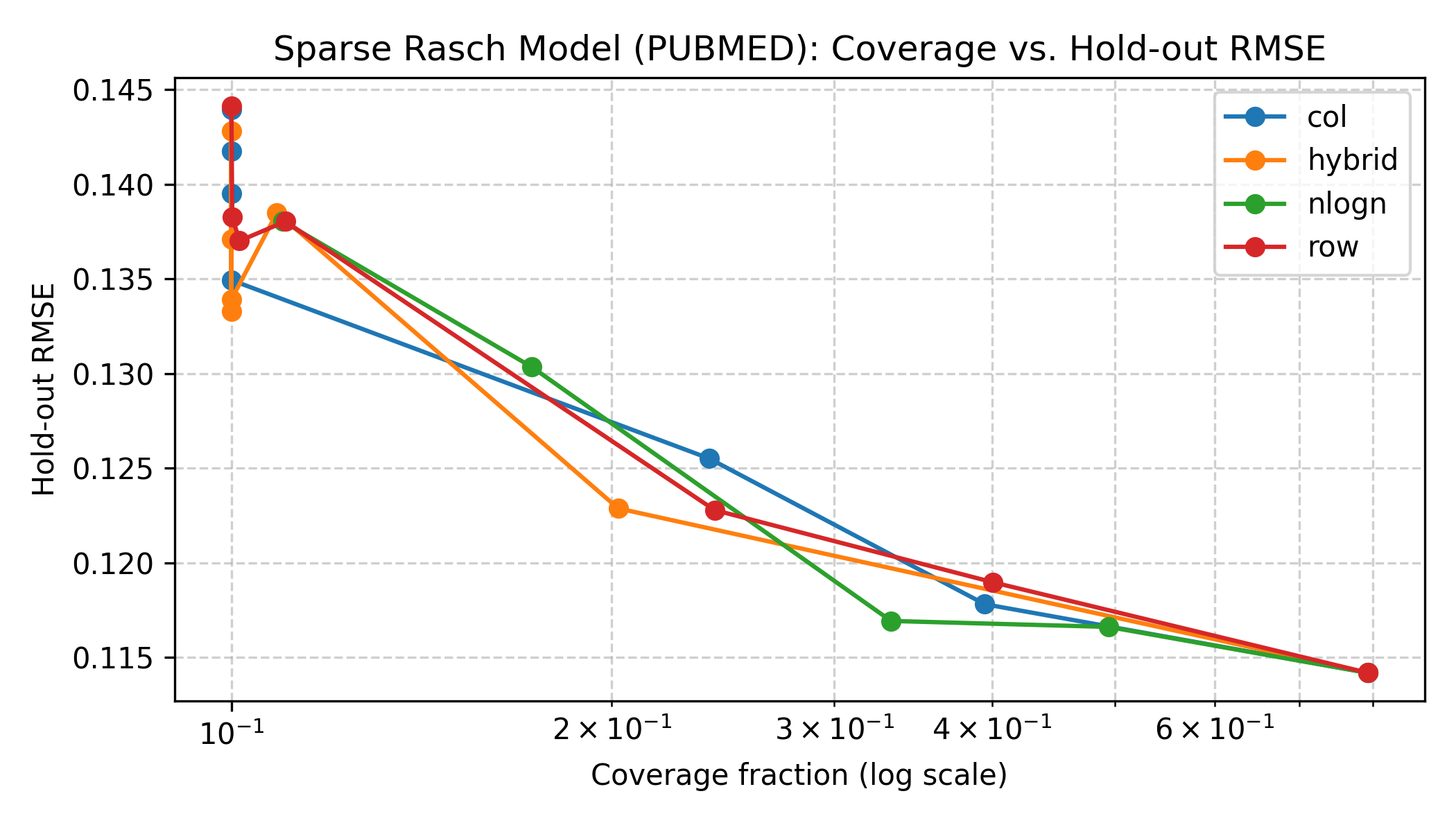}
  \caption{All sampling regimes}
  \label{fig:sparsity-all}
\end{subfigure}
\hfill
\begin{subfigure}{0.4\linewidth}
  \includegraphics[width=\linewidth]{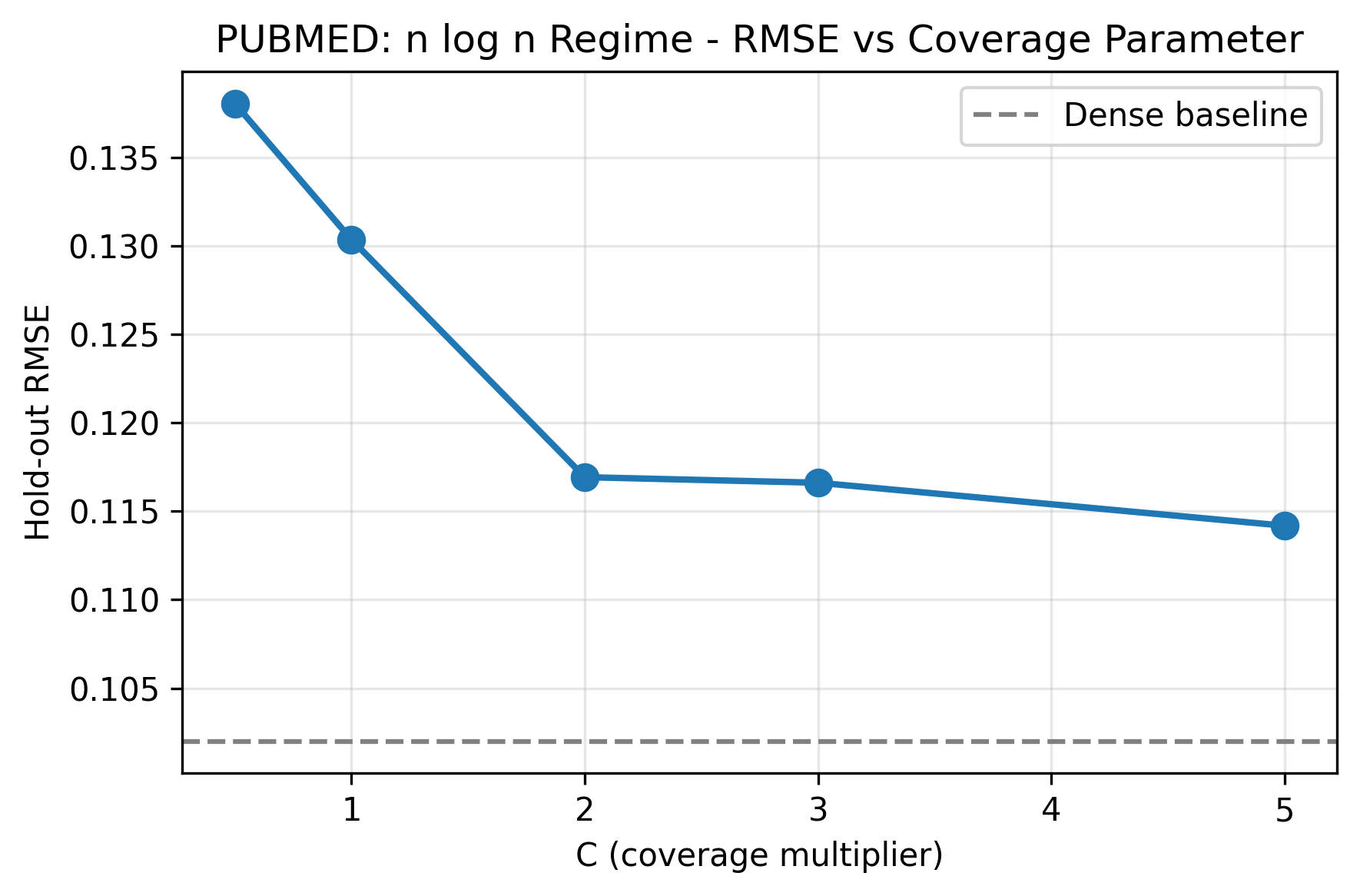}
  \caption{$(n\log n)$ regime detail}
  \label{fig:nlogn-sweep}
\end{subfigure}
\caption{Hold-out RMSE vs training coverage on PubMed (30×200). (a) The $(n\log n)$ regime (orange) achieves near-dense accuracy at 33\% coverage across all sampling strategies. (b) Zooming into $(n\log n)$: the sweet spot occurs at $C \in [2,3]$, balancing coverage and accuracy. Below $C=1$, performance degrades rapidly as connectivity becomes harder to maintain. All sparse regimes enforce $d$-core $\geq$ 3.}
\label{fig:sparsity-curve}
\end{figure}

The row and column sampling strategies perform similarly, both requiring approximately 26\% coverage to achieve RMSE below 0.12. The hybrid regime underperforms at comparable coverage levels, suggesting that structured missingness patterns are more favorable than random sparsity—likely because they guarantee minimum observations per entity. All regimes enforce $d$-core $\geq 3$ connectivity.

\subsection{Ranking Fidelity}

Table~\ref{tab:fidelity} shows that sparse recovery at 33\% coverage achieves minimal reconstruction error while perfectly preserving agent rankings:

\begin{table}[ht]
\centering
\caption{Fidelity metrics at 33\% coverage ($(n\log n)$ with $C=2$) 
compared to dense baseline on PubMed domain (30 agents × 200 items). 
Hold-out RMSE measures reconstruction accuracy on 20\% reserved test data. 
Ranking AUC tests whether per-agent average scores can distinguish faithful 
from problematic agents. We use all 30 agents for model fitting; the ranking 
evaluation focuses on a subset of 4 faithful and 15 problematic agents identified 
in prior work \citep{robertson2025measure}. 
95\% confidence intervals from 500 bootstrap samples.}
\label{tab:fidelity}
\begin{tabular}{lcc}
\toprule
Metric & Dense & Sparse (33\%) \\
\midrule
\textbf{Reconstruction Accuracy} & & \\
\quad Hold-out RMSE & 0.111 [0.107, 0.116] & 0.117 [0.113, 0.122] \\
\quad Relative increase & -- & +5.4\% \\
\midrule
\textbf{Ranking Fidelity} & & \\
\quad Spearman $\rho$ & -- & 0.972 [0.942, 0.986] \\
\quad Kendall $\tau$ & -- & 0.890 [0.833, 0.933] \\
\quad Ranking AUC & 0.983 [0.967, 0.992] & 0.967 [0.942, 0.983] \\
\bottomrule
\end{tabular}
\end{table}

Agent rankings are nearly perfectly preserved (Spearman $\rho = 0.972$ [0.942, 0.986], Kendall $\tau = 0.890$ [0.833, 0.933]), ensuring that relative model comparisons remain valid. The high Ranking AUC (0.967 [0.942, 0.983] for sparse vs 0.983 [0.967, 0.992] for dense) confirms that we can still distinguish high-quality from problematic agents despite 3× reduction in observations. The 5.4\% increase in hold-out RMSE (0.111 [0.107, 0.116] to 0.117 [0.113, 0.122]) represents minor reconstruction noise that does not affect downstream ranking tasks. The narrow confidence intervals show that these results are stable across different training samples.

\subsection{Cross-Domain Validation}

Table~\ref{tab:cross-domain} confirms that our findings generalize beyond summarization. All domains show consistent patterns: 3× coverage reduction with small absolute RMSE increases ($\Delta$ ranging from +0.003 to +0.006), and strong rank preservation (Spearman $\rho > 0.95$). Ranking AUC varies by domain (0.967 for PUBMED, 0.938 for OPUS, 0.646 for ICLR), reflecting inherent differences in agent discriminability rather than failures of sparse recovery—these values represent the ranking quality achievable even with dense evaluation in each domain.

\begin{table}[ht]
\centering
\caption{Cross-domain validation at 33\% coverage. All domains show minimal 
RMSE increase and strong rank preservation. Ranking AUC depends on the inherent 
discriminability of each domain's agent pool.}
\label{tab:cross-domain}
\begin{tabular}{lccccccc}
\toprule
Domain & Agents & Items & Dense & Sparse & RMSE $\Delta$ & Spearman & Ranking \\
       &        &       & RMSE  & RMSE   &                & $\rho$   & AUC \\
\midrule
PUBMED & 30 & 200 & 0.111 & 0.117 & +0.006 (+5\%) & 0.972 & 0.967 \\
OPUS & 31 & 186 & 0.123 & 0.126 & +0.003 (+2\%) & 0.983 & 0.938 \\
ICLR & 29 & 100 & 0.129 & 0.135 & +0.006 (+4\%) & 0.950 & 0.646 \\
\bottomrule
\end{tabular}
\end{table}

\subsection{Baseline Comparison}
\label{sec:baseline}

To validate that the additive constraint and identity link are essential design choices, we compare our clipped-linear model against four baselines at 33\% coverage on PubMed. Table~\ref{tab:baselines} shows results.

\begin{table}[ht]
\centering
\caption{Baseline comparison at 33\% coverage on PubMed (30 agents × 200 items). 
Clipped-linear (identity link) achieves best RMSE while non-additive UV factorization 
shows 61\% higher error. SVD baseline's poor reconstruction (0.172) demonstrates that 
simple mean imputation fails without the additive constraint.}
\label{tab:baselines}
\begin{tabular}{lccccccc}
\toprule
Method & RMSE & Spearman & Kendall & AUC \\
       &      & $\rho$   & $\tau$  &     \\
\midrule
\textbf{Clipped-Linear (Ours)} & \textbf{0.121} & 0.971 & 0.894 & 0.967 \\
Identity + Isotonic (calibrated monotone mapping) & 0.122 & 0.971 & 0.894 & 0.967\\
Rasch (Probit) & 0.122 & 0.965 & 0.876 & 0.983  \\
Rasch (Logit) & 0.123 & 0.964 & 0.871 & 0.983 \\
Nuclear Norm & 0.126 & 0.974 & 0.894 & 0.950 \\
SVD Baseline & 0.172 & 0.923 & 0.807 & 0.967\\
UV Factorization & 0.196 & 0.983 & 0.913 & 0.967 \\
\bottomrule
\end{tabular}
\end{table}

\textbf{Identity link achieves best reconstruction.} Our clipped-linear model with identity link achieves the lowest holdout RMSE (0.121), making it competitive with Rasch baselines. The isotonic regression variant (0.122) offers no measurable gain, confirming that a learned monotone calibration does not capture additional curvature. The identity link already linearizes the response space (Section~\ref{sec:link-ablation}).

\textbf{Additive constraint is essential.} The non-additive UV factorization baseline, which fits $S \approx U V^T$ without constraining to the form $\theta_i - b_j$, achieves RMSE of 0.196 compared to our clipped-linear model (0.121). The SVD baseline with mean imputation also performs poorly (0.172), showing that simple low-rank structure without the additive constraint fails to capture TVD-MI geometry. While both maintain reasonable correlation and AUC their poor reconstruction confirms that additivity captures core properties of TVD-MI data.

\textbf{Computational efficiency.} The Rasch methods are significantly faster than other baselines making them practical for large-scale evaluation.

\subsection{Judge Robustness}

To validate that our findings are not artifacts of a specific judge model, we repeated the analysis using Llama3-70b (70 billion parameters) as an alternative judge on a subset of PubMed data (100 items, 30 agents). Table~\ref{tab:judge-robustness} shows strong cross-judge agreement.

\begin{table}[ht]
\centering
\caption{Judge robustness: GPT-4o-mini vs Llama3-70b on PubMed. 
Strong agreement in agent rankings (Spearman $\rho$=0.872) shows that TVD-MI 
captures genuine quality differences. Both judges show identity achieving zero 
curl on predictions while curved links introduce consistent violations.}
\label{tab:judge-robustness}
\begin{tabular}{lcc}
\toprule
Metric & GPT-4o-mini & Llama3-70b \\
\midrule
\textbf{Cross-Judge Agreement} & & \\
\quad Agent ranking correlation & \multicolumn{2}{c}{$\rho$ = 0.872} \\
\quad Item difficulty correlation & \multicolumn{2}{c}{$\rho$ = 0.687} \\
\midrule
\textbf{Curl Statistics (Identity Link)} & & \\
\quad Median $|\Delta|$ (predictions) & 0.000 & 0.000 \\
\quad P95 $|\Delta|$ (predictions) & 0.000 & 0.000 \\
\quad Median $|\Delta|$ (raw data) & 0.129 & 0.129 \\
\quad P95 $|\Delta|$ (raw data) & 0.466 & 0.448 \\
\midrule
\textbf{Curl Statistics (Logit Link)} & & \\
\quad Median $|\Delta|$ (predictions) & 0.009 & 0.009 \\
\quad P95 $|\Delta|$ (predictions) & 0.072 & 0.070 \\
\quad Median $|\Delta|$ (raw data) & 0.438 & 0.431 \\
\quad P95 $|\Delta|$ (raw data) & 1.433 & 1.346 \\
\bottomrule
\end{tabular}
\end{table}

The high correlation in agent rankings (Spearman $\rho$ = 0.872) indicates genuine quality differences rather than judge-specific biases. Both judges show identity achieving zero median curl on predictions while curved links introduce consistent violations ($|\Delta|\approx 0.01$), confirming that TVD-MI's additive geometry is best preserved without transformation across judge models.

The moderate correlation in item difficulties ($\rho$ = 0.687) shows that different judges may find different items challenging, which is expected given model-specific failure modes. However, the consistency in agent rankings and link ablation results supports the robustness of our approach.

\section{Discussion}

\subsection{Why Does Additive Structure Emerge?}

The identity link works because TVD-MI scores exhibit additive structure in their raw space. This arises from the averaging operation: if pairwise TVD-MI scores have form $\theta_i + \theta_k - b_j$, then agent $i$'s average score is $\theta_i + \bar{\theta_{-i}} - b_j \approx \theta_i + \bar \theta - b_j$, preserving additivity.

Traditional IRT's logistic links are designed for binary outcomes generated by latent Gaussian processes and maximize Shannon entropy. However, TVD-MI has its own natural entropy measure: Lemma~\ref{lem:tvd-gini} shows that $\text{TVD-MI}(X;X) = \text{Gini}(X)$, establishing Gini (Rényi-2) as the principled entropy for total variation distance. Applying logistic links—which maximize Shannon entropy—to TVD-MI introduces unnecessary distortion: identity yields zero median curl on predictions while curved links produce consistent violations, Table~\ref{tab:data-curl}). The clipped-linear model preserves TVD-MI's natural geometry by maximizing its native entropy measure.

\subsection{Limitations}

Sparse recovery requires: (1) meaningful quality variation among agents, (2) $d$-core $\geq 3$ connectivity (difficult to maintain below 15\% coverage), and (3) approximate additivity. Domains with strong agent-item interactions will show larger curl and degraded recovery; our cross-domain validation suggests this is rare for current general-purpose LLMs.

\subsection{Practical Recommendations}

Based on our empirical findings, we offer the following guidance for practitioners. \textbf{Use $(n\log n)$ sampling with $C \in [1.5,3]$:} This regime consistently achieves the best coverage-accuracy tradeoff, requiring only 30-40\% of full evaluation cost while maintaining fidelity. \textbf{Validate with discrete integrability test:} Before committing to sparse evaluation, test additivity on a small dense subset using the identity link. Median $|\Delta| < 0.2$ indicates suitable structure for sparse recovery. \textbf{Use identity link for TVD-MI:} Our ablation study (Section~4.3) shows that the identity link achieves 2--7$\times$ lower curl than logistic alternatives, preserving TVD-MI's additive structure.

\section{Reproducibility}

All code, data preprocessing scripts, experimental masks, and analysis are available at \url{https://github.com/zrobertson466920/sparse-tvd-irt}. Response data and LLM judge data are from \citet{robertson2025measure}.

\section{Conclusion}

We show that TVD-MI averaging naturally produces additive structure in the raw $[-1,1]$ space, which traditional logistic links destroy. Identity link preserves this geometry, while probit/logit introduce 2–7× higher violations across three domains. We show this enables up to 3$\times$ fewer evaluations than full dense (33\% coverage) with marginal increase in the downstream peer-prediction task.

Our clipped-linear model, derived from Gini entropy maximization, preserves TVD-MI's natural geometry and generalizes beyond LLM evaluation to any bounded-response domain where scores arise from averaging operations. The discrete integrability test is a simple diagnostic for validating link choices empirically, an alternative to the default use of logistic links in psychometrics for non-traditional response formats.

\newpage

\bibliographystyle{plainnat}
\bibliography{example_ref}

\begin{thebibliography}{17}
\providecommand{\natexlab}[1]{#1}
\providecommand{\url}[1]{\texttt{#1}}
\expandafter\ifx\csname urlstyle\endcsname\relax
  \providecommand{\doi}[1]{doi: #1}\else
  \providecommand{\doi}{doi: \begingroup \urlstyle{rm}\Url}\fi

\bibitem[Andersen(1977)]{andersen1977sufficient}
Erling~B Andersen.
\newblock Sufficient statistics and latent trait models.
\newblock \emph{Psychometrika}, 42\penalty0 (1):\penalty0 69--81, 1977.

\bibitem[Candes and Recht(2012)]{candes2012exact}
Emmanuel Candes and Benjamin Recht.
\newblock Exact matrix completion via convex optimization.
\newblock \emph{Communications of the ACM}, 55\penalty0 (6):\penalty0 111--119, 2012.

\bibitem[Castleman et~al.(2025)Castleman, Nadeem, Namjoshi, and Liu]{castleman2025rethinking}
Jane Castleman, Nimra Nadeem, Tanvi Namjoshi, and Lydia~T Liu.
\newblock Rethinking math benchmarks for llms using irt.
\newblock \emph{Proceedings of Machine Learning Research}, 273:\penalty0 66--82, 2025.

\bibitem[Chiang et~al.(2024)Chiang, Zheng, Sheng, Angelopoulos, Li, Li, Zhu, Zhang, Jordan, Gonzalez, et~al.]{chiang2024chatbot}
Wei-Lin Chiang, Lianmin Zheng, Ying Sheng, Anastasios~Nikolas Angelopoulos, Tianle Li, Dacheng Li, Banghua Zhu, Hao Zhang, Michael Jordan, Joseph~E Gonzalez, et~al.
\newblock Chatbot arena: An open platform for evaluating llms by human preference.
\newblock In \emph{Forty-first International Conference on Machine Learning}, 2024.

\bibitem[Dasgupta et~al.(2013)Dasgupta, Ghosh, and Munagala]{dasgupta2013crowdsourced}
Anirban Dasgupta, Arpita Ghosh, and Kamesh Munagala.
\newblock Crowdsourced judgement elicitation with endogenous proficiency.
\newblock \emph{Proceedings of the 22nd International Conference on World Wide Web}, pages 319--330, 2013.

\bibitem[Ding et~al.(2024)Ding, Deng, Choo, Wu, Agrawal, Schwarzschild, Zhou, Goldstein, Langford, Anandkumar, et~al.]{ding2024easy2hard}
Mucong Ding, Chenghao Deng, Jocelyn Choo, Zichu Wu, Aakriti Agrawal, Avi Schwarzschild, Tianyi Zhou, Tom Goldstein, John Langford, Animashree Anandkumar, et~al.
\newblock Easy2hard-bench: Standardized difficulty labels for profiling llm performance and generalization.
\newblock \emph{Advances in Neural Information Processing Systems}, 37:\penalty0 44323--44365, 2024.

\bibitem[Glenn et~al.(1950)]{glenn1950verification}
W~Brier Glenn et~al.
\newblock Verification of forecasts expressed in terms of probability.
\newblock \emph{Monthly weather review}, 78\penalty0 (1):\penalty0 1--3, 1950.

\bibitem[Liu et~al.(2025)Liu, Bhandari, and Pardos]{liu2025leveraging}
Yunting Liu, Shreya Bhandari, and Zachary~A Pardos.
\newblock Leveraging llm respondents for item evaluation: A psychometric analysis.
\newblock \emph{British Journal of Educational Technology}, 56\penalty0 (3):\penalty0 1028--1052, 2025.

\bibitem[McCullagh(2019)]{mccullagh2019generalized}
Peter McCullagh.
\newblock \emph{Generalized linear models}.
\newblock Routledge, 2019.

\bibitem[Prelec(2004)]{prelec2004bayesian}
Dra{\v{z}}en Prelec.
\newblock A bayesian truth serum for subjective data.
\newblock \emph{Science}, 306\penalty0 (5695):\penalty0 462--466, 2004.

\bibitem[Rasch(1993)]{rasch1993probabilistic}
Georg Rasch.
\newblock \emph{Probabilistic Models for Some Intelligence and Attainment Tests}.
\newblock ERIC, 1993.

\bibitem[Robertson and Koyejo(2025)]{robertson2025measure}
Zachary Robertson and Sanmi Koyejo.
\newblock Let's measure information step-by-step: Llm-based evaluation beyond vibes.
\newblock \emph{arXiv preprint arXiv:2508.05469}, 2025.

\bibitem[Sriperumbudur et~al.(2009)Sriperumbudur, Fukumizu, Gretton, Sch{\"o}lkopf, and Lanckriet]{sriperumbudur2009integral}
Bharath~K Sriperumbudur, Kenji Fukumizu, Arthur Gretton, Bernhard Sch{\"o}lkopf, and Gert~RG Lanckriet.
\newblock On integral probability metrics,$\backslash$phi-divergences and binary classification.
\newblock \emph{arXiv preprint arXiv:0901.2698}, 2009.

\bibitem[Tsybakov(2008)]{tsybakov2008nonparametric}
Alexandre~B. Tsybakov.
\newblock Nonparametric estimators.
\newblock In \emph{Introduction to Nonparametric Estimation}, pages 1--76. Springer, 2008.

\bibitem[Xu et~al.(2025)Xu, Lu, Schoenebeck, and Kong]{xubenchmarking}
Shengwei Xu, Yuxuan Lu, Grant Schoenebeck, and Yuqing Kong.
\newblock Benchmarking llms' judgments with no gold standard.
\newblock In \emph{The Thirteenth International Conference on Learning Representations}, 2025.

\bibitem[Yuan et~al.(2021)Yuan, Wu, and Zhang]{yuan2021gini}
Ye~Yuan, Liji Wu, and Xiangmin Zhang.
\newblock Gini-impurity index analysis.
\newblock \emph{IEEE Transactions on Information Forensics and Security}, 16:\penalty0 3154--3169, 2021.

\bibitem[Zhou et~al.(2025)Zhou, Huang, Zhao, Han, Wang, Chen, Yang, Bao, Dong, Xu, et~al.]{zhou2025lost}
Hongli Zhou, Hui Huang, Ziqing Zhao, Lvyuan Han, Huicheng Wang, Kehai Chen, Muyun Yang, Wei Bao, Jian Dong, Bing Xu, et~al.
\newblock Lost in benchmarks? rethinking large language model benchmarking with item response theory.
\newblock \emph{arXiv preprint arXiv:2505.15055}, 2025.

\end{thebibliography}

\appendix
\newpage

\section{Agent Configurations}
\label{app:agent-configs}

We evaluate 30 agent configurations across three domains:

\textbf{Faithful baselines (4 agents):} GPT-4, Claude-3-Opus, Gemini-1.5-Pro, and Llama-3-70b with standard prompts.

\textbf{Stylistic variants (6 agents):} Temperature variations (0.3, 0.7, 1.0), length constraints (concise, verbose), and formatting modifications.

\textbf{Strategic distortions (15 agents):} Paraphrasing, synonym substitution, sentence reordering, adding filler content, selective omission, and combinations thereof.

\textbf{Low-effort responses (5 agents):} Truncated outputs, random text, template responses, and minimal-effort completions.

This agent pool provides sufficient quality variation to test discriminative evaluation while including realistic failure modes.

\section{TVD-MI Score Construction}
\label{app:tvdmi-construction}

For each agent pair $(i, j)$ and item $k$, we compute TVD-MI scores through binary classification. An LLM judge (GPT-4o-mini) distinguishes whether response pairs come from the same agent (positive) or different agents (negative), yielding true and false positive rates (TPR, FPR). The TVD-MI score for agent $i$ on item $k$ is:
\begin{equation}
s_{ik} = \frac{1}{K-1} \sum_{j \neq i} (\text{TPR}_{ijk} - \text{FPR}_{ijk}),
\end{equation}
where $K=30$ is the number of agents. This averages the discriminative signal across all other agents, producing matrices $S \in [-1, 1]^{K \times J}$.

\textbf{Train-test independence protocol.} We construct the holdout set by randomly selecting 20\% of all possible agent-item pairs $(i,k)$ before computing any TVD-MI scores. For training, we compute $s_{ik}$ using only the pairwise comparisons $(i,j,k)$ where $(i,k)$ is not in the holdout set. Critically, if agent $i$ on item $k$ is held out, we exclude \emph{all} pairwise terms $(\text{TPR}_{ijk} - \text{FPR}_{ijk})$ for any $j$ from the averaging operation. This ensures the training score $s_{ik}$ for held-out pairs is never computed, preventing test data leakage.

\section{Data Characteristics}
\label{app:data-characteristics}

The resulting matrices exhibit three properties that inform our modeling:

\textbf{Boundary saturation.} Approximately 2-3\% of entries saturate at $\pm 1$ across domains, occurring when agents achieve perfect discrimination (expert agents on easy items) or complete failure (degenerate responses). We clip scores to $[-0.99, 0.99]$ before applying the link function, preventing undefined values while preserving rank ordering.

\textbf{Score distributions.} Scores center around 0.18 (SD = 0.31) with heavier-than-Gaussian tails, capturing both typical moderate discrimination and exceptional performance/failure modes.

\textbf{Natural sparsity.} Computational constraints in the original evaluation left 6-13\% of entries unevaluated across domains, providing natural held-out test data. This natural sparsity pattern differs from our structured sampling regimes (Section~5) but shows that the evaluation framework can handle incomplete observations.

\section{Sampling and Protocol Details}
\label{app:sampling-protocol}

\subsection{Evaluation Framework}

We evaluate sparse recovery through a train-test protocol designed to mirror practical deployment scenarios. \textbf{Holdout construction:} We reserve 20\% of all possible agent-item pairs as a disjoint test set, ensuring no overlap with training observations. This holdout remains fixed across all experiments. \textbf{Sparse sampling regimes:} From the remaining 80\% of pairs, we construct training masks under different sparsity patterns (described below), each enforcing $d$-core $\geq 3$ connectivity. \textbf{Model fitting:} We fit the clipped-linear model (Eq.~\ref{eq:rasch-objective}) on the sparse training data with regularization $\lambda = 10^{-6}$. \textbf{Evaluation:} We measure hold-out RMSE and downstream fidelity metrics on the reserved test pairs.

\subsection{Structured Sparsity Regimes}

We investigate four sampling strategies that reflect different practical constraints:

\paragraph{Row sampling ($\alpha$-coverage).} Each agent observes a random $\alpha$-fraction of items. This models scenarios where agents have limited evaluation budgets. We test $\alpha \in \{0.15, 0.30, 0.45\}$.

\paragraph{Column sampling ($\beta$-coverage).} Each item is evaluated by a random $\beta$-fraction of agents. This models scenarios where items have evaluation quotas. We test $\beta \in \{0.15, 0.30, 0.45\}$.

\paragraph{Hybrid sampling.} Each pair $(i,j)$ is observed independently with probability $\alpha \cdot \beta$. This provides a baseline for unstructured sparsity.

\paragraph{Efficient $(n \log n)$ sampling.} We observe $|\Omega| = C(K + J)\log(K + J)$ randomly selected pairs, motivated by matrix completion theory. We sweep $C \in \{0.5, 1.0, 2.0, 3.0, 5.0\}$ to identify the coverage-accuracy frontier.

For all regimes, we enforce $d$-core $\geq 3$ connectivity by iteratively adding random observations to under-connected nodes until the constraint is satisfied.

\subsection{Fidelity Metrics}

We evaluate sparse recovery through two complementary metrics:

\paragraph{Reconstruction accuracy (RMSE).} We measure hold-out root mean squared error between predicted and observed scores on the 20\% reserved test set. This quantifies how well the sparse model recovers the underlying score matrix.

\paragraph{Ranking fidelity.} Following the evaluation protocol in \citet{robertson2025measure}, we partition agents into "faithful" (high-quality) and "problematic" (distorted/low-effort) groups. We then evaluate ranking preservation through:
\begin{itemize}
\item \textbf{Rank correlation:} Spearman's $\rho$ and Kendall's $\tau$ between agent abilities $\theta$ estimated from dense versus sparse data
\item \textbf{Ranking AUC:} For each agent, we average their scores across all evaluated items. We then compute AUC for the binary classification task: can we distinguish faithful from problematic agents using these per-agent average scores?
\end{itemize}

These metrics directly test whether sparse recovery preserves the core utility of TVD-MI evaluation: the ability to rank and compare models.

\subsection{Statistical Inference}

We quantify uncertainty through bootstrap resampling with 500 iterations. For each bootstrap sample, we:
\begin{enumerate}
\item Resample training pairs with replacement (separately for dense and sparse masks)
\item Refit the clipped-linear model on the bootstrap training data
\item Evaluate on the original (non-resampled) holdout set
\item Compute all metrics (RMSE, rank correlations, AUC)
\end{enumerate}

This procedure captures model uncertainty arising from finite training samples. By resampling the \emph{training} data rather than the test set, we obtain realistic confidence intervals that reflect variability in parameter estimates. We report 95\% confidence intervals as the 2.5th and 97.5th percentiles of the bootstrap distribution.

\section{TVD-MI and Gini Entropy Connection}
\label{app:tvd-gini-connection}

We first establish that Gini entropy is the natural entropy measure for total variation distance.

\begin{lemma}[TVD Self-Information Equals Gini Entropy]
\label{lem:tvd-gini}
For a discrete random variable $X$ with distribution $P_X$, the TVD mutual information between $X$ and itself satisfies:
\begin{equation}
\text{TVD-MI}(X;X) = \text{Gini}(X) = 1 - \sum_x P_X(x)^2
\end{equation}
\end{lemma}

\begin{proof}
When $Y = X$, the joint distribution $P_{XX}$ concentrates on the diagonal: $P_{XX}(x,x) = P_X(x)$, $P_{XX}(x,y) = 0$ for $x \neq y$. By definition:
\begin{align}
\text{TVD-MI}(X;X) &= \frac{1}{2}\sum_{x,y} |P_{XX}(x,y) - P_X(x)P_X(y)| \\
&= \frac{1}{2}\left[\sum_x |P_X(x) - P_X(x)^2| + \sum_{x \neq y} P_X(x)P_X(y)\right]
\end{align}
The diagonal terms give $\sum_x P_X(x)(1 - P_X(x))$ and the off-diagonal terms give $1 - \sum_x P_X(x)^2$. Combining:
\begin{align}
\text{TVD-MI}(X;X) &= \frac{1}{2}\left[\sum_x P_X(x) - \sum_x P_X(x)^2 + 1 - \sum_x P_X(x)^2\right] \\
&= 1 - \sum_x P_X(x)^2 = \text{Gini}(X)
\end{align}
\end{proof}

Maximizing Gini entropy is equivalent to maximizing TVD self-information, the natural entropy measure for TVD-MI scores..

\section{Maximum Gini Entropy Derivation}
\label{app:gini-derivation}

We now provide the complete derivation showing that Gini entropy maximization yields the clipped-linear model with identity link.

\subsection{Primal Formulation}

We validate empirically through discrete integrability tests (Section~\ref{sec:link-ablation} of main text) that TVD-MI scores exhibit approximate additive structure. We think of this as (curl) deviation from an idealized assumption. 

\textbf{Assumption:} We want to find the best additive approximation $s_{ij} \in [-1,1]$ for all agent-item pairs, subject to:
\begin{enumerate}
\item \textbf{Observational consistency:} $s_{ij} = t_{ij}$ for all $(i,j) \in \Omega$ (observed pairs)
\item \textbf{Discrete integrability:} $s_{ij} - s_{i'j} = s_{ij'} - s_{i'j'}$ for all rectangles
\end{enumerate}

Using Gini (Rényi-2) entropy as the projection criterion, we minimize:
\begin{equation}
\min_{s} \sum_{(i,j)} s_{ij}^2
\end{equation}
subject to $|s_{ij}| \leq 1$ and $s_{ij} = t_{ij}$ for $(i,j) \in \Omega$.

By the discrete integrability lemma (Lemma~\ref{lem:integrability} below), the rectangle constraints are satisfied if and only if $s_{ij} = \theta_i - b_j$ for some parameters $\theta_i, b_j$. This a derivation from first principles of TVD-MI based entropy is Gini entropy (Appendix~\ref{app:tvd-gini-connection}) which reduces to projection onto the additive manifold under a quadratic objective.

\subsection{Discrete Integrability Lemma}

\begin{lemma}[Discrete Integrability]
\label{lem:integrability}
A matrix $s_{ij}$ satisfies the rectangle constraints
\begin{equation}
s_{ij} - s_{i'j} = s_{ij'} - s_{i'j'} \quad \forall (i,i',j,j')
\end{equation}
if and only if there exist parameters $\theta_i, b_j \in \mathbb{R}$ such that $s_{ij} = \theta_i - b_j$.
\end{lemma}

\begin{proof}
\textbf{Sufficiency:} If $s_{ij} = \theta_i - b_j$, then:
\begin{align}
s_{ij} - s_{i'j} - s_{ij'} + s_{i'j'} &= (\theta_i - b_j) - (\theta_{i'} - b_j) - (\theta_i - b_{j'}) + (\theta_{i'} - b_{j'}) \\
&= \theta_i - \theta_{i'} - \theta_i + \theta_{i'} = 0
\end{align}

\textbf{Necessity:} Fix a reference pair $(i_0, j_0)$. Define:
\begin{equation}
\theta_i := s_{ij_0} - s_{i_0 j_0}, \quad b_j := -s_{i_0 j}
\end{equation}

Then for any $(i,j)$, by the rectangle constraint with $(i, i_0, j_0, j)$:
\begin{align}
s_{ij} - s_{i_0 j} &= s_{ij_0} - s_{i_0 j_0} \\
s_{ij} &= (s_{ij_0} - s_{i_0 j_0}) + s_{i_0 j} = \theta_i - b_j
\end{align}
\end{proof}

\subsection{Dual Formulation via Convex Conjugate}

To derive the dual, we compute the convex conjugate of $\frac{1}{2}s^2$ subject to $|s| \leq 1$:
\begin{equation}
\phi^*(u) := \sup_{|s| \leq 1} \left\{ us - \frac{1}{2}s^2 \right\}
\end{equation}

For $|u| \leq 1$, the supremum is attained at $s = u$ (interior), giving $\phi^*(u) = u^2 - \frac{1}{2}u^2 = \frac{1}{2}u^2$.

For $|u| > 1$, the supremum is attained at $s = \text{sign}(u)$ (boundary), giving $\phi^*(u) = |u| - \frac{1}{2}$.

Thus:
\begin{equation}
\phi^*(u) = \begin{cases}
\frac{1}{2}u^2, & |u| \leq 1 \\
|u| - \frac{1}{2}, & |u| > 1
\end{cases}
\end{equation}

This is the Huber loss with threshold 1.

\subsection{Lagrangian, Integrability, and the Identity Link}

Let $\phi(s)=\tfrac12 s^2+\iota_{[-1,1]}(s)$, where $\iota_C$ is the indicator of a constraint set $C$.  
Let $D$ be the rectangle (curl) operator so that $Ds=0$ iff $s_{ij}=\theta_i-b_j$ for some $(\theta,b)$ (Lemma~\ref{lem:integrability}).  
We solve the projection
\[
\min_{s}\ \sum_{i,j}\phi(s_{ij})\quad\text{s.t.}\quad P_\Omega s=t,\ \ Ds=0.
\]
Introducing Lagrange multipliers $\lambda$ for $P_\Omega s=t$ and $\mu$ for $Ds=0$, and defining
$u:=P_\Omega^\top\lambda - D^\top\mu$, the Lagrangian is
\[
\mathcal L(s,\lambda,\mu)
= \sum_{i,j}\phi(s_{ij})-\langle u,s\rangle+\langle \lambda,t\rangle.
\]
Minimizing over $s$ is separable entrywise and yields the conjugate $\phi^*$:
\[
\inf_{s}\mathcal L
=\langle \lambda,t\rangle-\sum_{i,j}\phi^*(u_{ij}),\qquad
\phi^*(u)=
\begin{cases}
\tfrac12 u^2,&|u|\le1,\\[4pt]
|u|-\tfrac12,&|u|>1.
\end{cases}
\]
Hence the dual is
\[
\max_{\lambda,\mu}\ 
\langle \lambda,t\rangle-\sum_{i,j}\phi^*\!\big((P_\Omega^\top\lambda-D^\top\mu)_{ij}\big).
\tag{D}
\label{eq:dual}
\]

\paragraph{Projection interpretation.}
For fixed $\lambda$, the quantity inside the sum depends on 
$u=P_\Omega^\top\lambda-D^\top\mu$.  
Because $\mathrm{im}(D^\top)=(\ker D)^\perp$, varying $\mu$ moves $u$ within the affine space 
$P_\Omega^\top\lambda+\mathrm{im}(-D^\top)$.
Since $\phi^*$ is convex and radially nondecreasing in $|u|$, minimizing
$\sum_{i,j}\phi^*(u_{ij})$ over this affine space selects the element of smallest Euclidean norm i.e., the orthogonal projection of $P_\Omega^\top\lambda$ onto $\ker D$.  
Thus the optimizer satisfies
\[
u^\star = P_\Omega^\top\lambda - D^\top\mu^\star \in \ker D,
\]
so there exist parameters $\theta,b$ such that
$u^\star_{ij}=\theta_i-b_j$.

\paragraph{Recovering the primal in additive form.}
The KKT stationarity condition for $s$ gives
$u_{ij}\in\partial\phi(s_{ij})$, implying $u_{ij}=s_{ij}$ when $|s_{ij}|<1$
and $u_{ij}$ has the same sign with $|u_{ij}|\ge1$ when $|s_{ij}|=1$.
Together with primal feasibility $P_\Omega s=t$ and $Ds=0$, we obtain
$s_{ij}=\theta_i-b_j$ up to clipping at the box boundaries.

Relaxing the hard constraint $P_\Omega s=t$ to a quadratic penalty for noisy data yields
\[
\min_{\theta,b}
\sum_{(i,j)\in\Omega}\tfrac12\big(t_{ij}-(\theta_i-b_j)\big)^2
+\sum_{i,j}\iota_{[-1,1]}(\theta_i-b_j)
+\lambda(\|\theta\|_2^2+\|b\|_2^2),
\]
which is the box-constrained least-squares objective 
(Eq.~\ref{eq:rasch-objective}) with the \emph{identity} link.

\subsection{Comparison to Shannon Entropy}

For comparison, Shannon entropy $H(p) = -p\log p - (1-p)\log(1-p)$ yields the logistic link. The first-order condition for maximizing Shannon entropy subject to linear constraints gives:
\begin{equation}
\frac{\partial H}{\partial p} = -\log p + \log(1-p) = \lambda \quad \Rightarrow \quad p = \frac{1}{1 + e^{-\lambda}} = \text{logit}^{-1}(\lambda)
\end{equation}

This motivates the logit link $\text{logit}(p) = \log(p/(1-p))$ in traditional IRT. However, for TVD-MI scores that are already bounded and arise from averaging (not thresholding), the Gini entropy's quadratic structure is more natural and preserves additivity without transformation.

\end{document}